%% file: main.tex
\documentclass[monochrome,10pt]{article}

\usepackage[preprint]{tmlr}

\input{preamble}

\usepackage{url}

\title{HiAP: A Multi-Granular Stochastic Auto-Pruning Framework for Vision Transformers}

\author{\name Andy Li \\
      \addr School of Natural and Computing Sciences, University of Aberdeen, UK
      \AND
      \name Aiden Durrant \\
      \addr 
      School of Computing Sciences, University of East Anglia, UK
      \AND
      \name Milan Markovic \\
      \addr School of Natural and Computing Sciences, University of Aberdeen, UK
      \AND
      \name Georgios Leontidis \\
      \addr 
      Department of Physics and Technology, UiT The Arctic University of Norway, Norway}

\def\month{MM}
\def\year{YYYY}
\def\openreview{\url{https://openreview.net/forum?id=XXXX}}

\begin{document}
\maketitle

\input{sec/abstract_revision}
\input{sec/paper_revision}

\bibliography{main}
\bibliographystyle{tmlr}

\newpage
\input{supp_revision}

\end{document}

%% file: preamble.tex
\usepackage[monochrome]{xcolor}
\usepackage{graphicx}
\usepackage{amsmath}
\usepackage{amssymb}
\usepackage{booktabs}
\usepackage{caption}
\usepackage{subcaption}

\usepackage[accsupp]{axessibility}
\usepackage{algorithm}
\usepackage{algpseudocode}
\usepackage{amsthm}
\usepackage{colortbl}

\newtheorem{lemma}{Lemma}
\newtheorem{proposition}{Proposition}


%% file: sec/abstract_revision.tex
\begin{abstract}
Vision Transformers require significant computational resources and memory bandwidth, severely limiting their deployment on resource-constrained hardware. Most structured pruning methods reduce theoretical cost effectively, yet they typically operate at a single structural granularity and depend on multi-stage pipelines with importance ranking, auxiliary solvers or post-hoc magnitude thresholding, followed by a separate fine-tuning phase to recover accuracy. We propose Hierarchical Auto-Pruning (HiAP), which casts ViT pruning as a single budget-aware learning problem and jointly allocates sparsity across four granularities in one end-to-end phase. HiAP introduces stochastic Gumbel-Sigmoid gates at macro level (attention heads and FFN blocks) and micro level (intra-head dimensions and FFN neurons), and trains them against the task loss together with an analytical MAC cost term. \textcolor{blue}{Two penalty weights set how much compute is removed, while the gates gradually harden into a dense, smaller sub-network at convergence.} It does not require importance heuristics, ranking metrics, auxiliary solvers or secondary fine-tuning. On ImageNet, HiAP compresses DeiT-Base to 7.4G MACs at 80.88\% top-1 \textcolor{blue}{(80.40\% at 7.8G without distillation)} and DeiT-Small to 3.1G at 79.33\%, competitive with substantially more complex pipelines at matched compute. The structurally pruned network can be accelerated natively on stock kernels, and more than 90\% of the theoretical MAC reduction is realized as measured throughput \textcolor{blue}{at batch 256} on an A100.
\end{abstract}

%% file: sec/paper_revision.tex
\section{Introduction}
Vision Transformers (ViTs) \citep{dosovitskiy2020image} are a dominant architecture in computer vision, but their high computation and memory costs make them difficult to deploy on resource-constrained devices. Model pruning resolves this by removing redundant parts of a network, reducing its parameters and FLOPs. While unstructured pruning (removing individual weights) offers the highest flexibility and theoretical compression ratio, the resulting irregular sparsity patterns require specialized hardware to achieve an actual acceleration. Structured pruning, which removes entire components such as attention heads or feed-forward neurons, has become the preferred method, as it results in a smaller, dense sub-network that natively accelerates on standard hardware.

Despite recent advances in structured pruning for ViTs, many methods still face two recurring limitations. First, single-granularity pruning leaves one bottleneck unaddressed. Methods that prune only micro-structures (e.g., intra-head dimensions) \citep{chavan2022vision} reduce FLOPs but preserve the network's overall depth and the number of attention heads.
Yet latency and energy on modern hardware are often dominated not by compute but by memory bandwidth: transferring data between DRAM and SRAM can cost orders of magnitude more energy than the corresponding computation \citep{han2015deep}. ViTs are particularly exposed to this cost since preserving every layer and materializing each $N \times N$ attention map requires repeated weight loads and large intermediate transfers \citep{dao2022flashattention}. Conversely, methods that exclusively prune macro structures (e.g., entire heads, blocks) \citep{liu2024updp} can remove this overhead, but often risk more performance degradation as they lose the network's representation capacity to a greater extent. Addressing both bottlenecks requires pruning macro and micro structures jointly. Second, multi-stage pipelines are complex and hard to reproduce. Many strong methods decompose pruning into separate searching, importance ranking, post-hoc thresholding, and retraining stages \citep{yin2023gohsp,zheng2022savit,chavan2022vision}, each adding its own design choices, hyperparameters, and auxiliary machinery (graph evaluations, solvers, scheduled mask updates etc). These stack mechanisms for incremental gains and are difficult to implement, extend, and reason about. Table~\ref{tab:comparison} summarizes which granularities recent methods target and what machinery each requires to enforce a budget.

We propose Hierarchical Auto-Pruning (HiAP), which casts pruning as a single-shot, \emph{budget-aware} learning problem solved in one end-to-end phase. Rather than hand-designing per-layer sparsity ratios or running a separate search-and-retrain pipeline, we set two penalty weights on the expected compute, and HiAP learns the model weights together with a per-layer, per-granularity sparsity allocation. Concretely, we add a differentiable analytical MAC cost term and structural feasibility penalties, which prevent layer collapse to the task loss, and relax the discrete keep/prune decisions with Gumbel-Sigmoid gates so gradients flow end-to-end. Gates act at two levels - macro (entire attention heads and FFN blocks) and micro (intra-head dimensions and FFN neurons), and harden into a dense, smaller sub-network as the temperature anneals, with no importance heuristics, ranking metrics, auxiliary solvers, post-hoc \textcolor{blue}{magnitude} thresholding, or secondary fine-tuning. Because the cost model attributes hardware penalties to individual structures, the sparsity allocation is \emph{heterogeneous}, and differs across models. \textcolor{blue}{On DeiT-Base the network gives up attention heads while keeping every FFN block, and on DeiT-Small it closes whole FFN blocks and keeps proportionally more heads (Sec.~\ref{subsec:structure_analysis}).} Figure \ref{fig:hiap_overview} illustrates the gating mechanism and how differentiable cost flows back to the gate logits.

\begin{figure}[t]
\centering
\includegraphics[width=\linewidth,height=0.72\textheight,keepaspectratio]{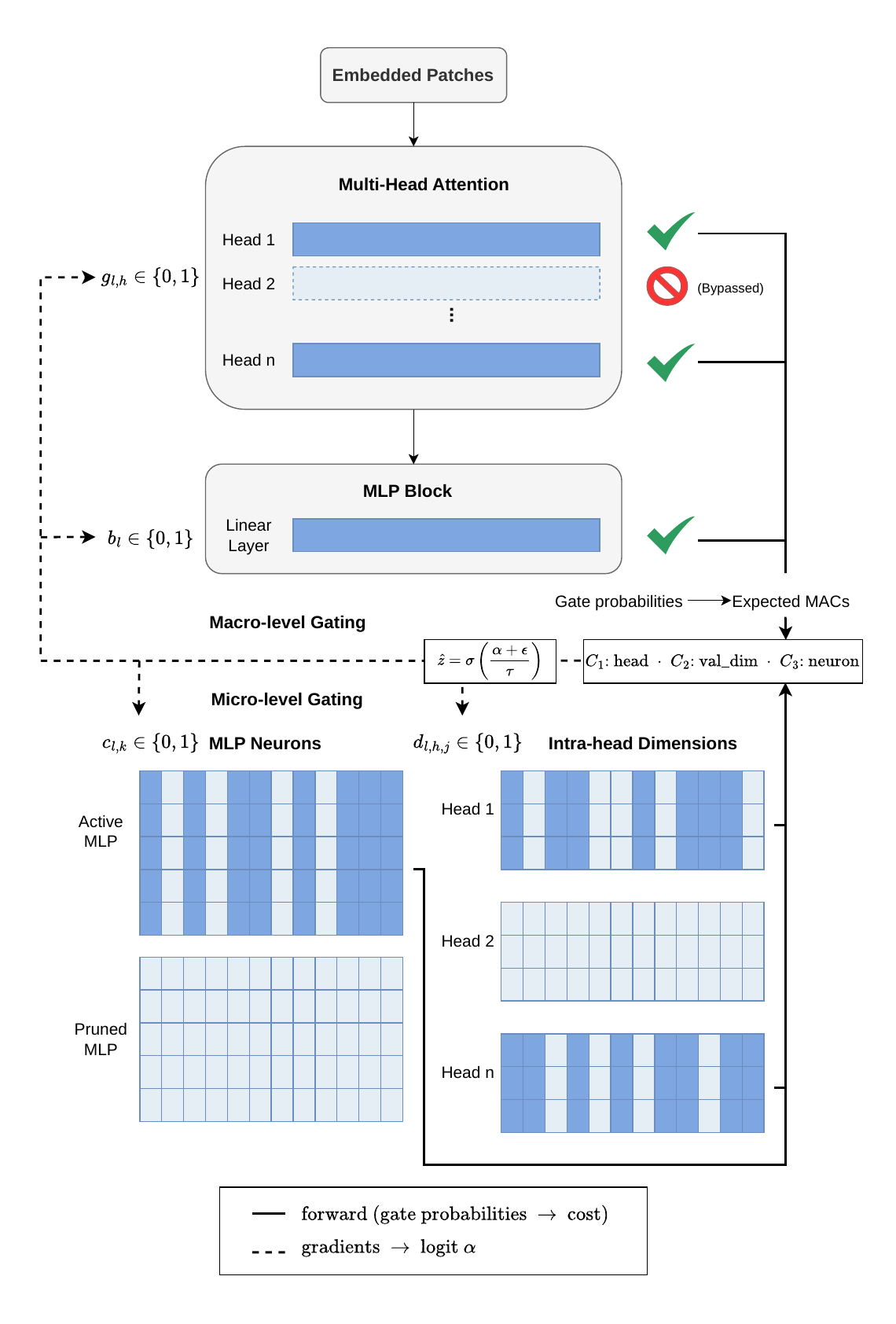}
\caption{Overview of the Hierarchical Auto-Pruning (HiAP) framework applied to a standard Vision Transformer block. The architecture's topology is governed by learnable Gumbel-Sigmoid gates operating at two distinct granularities. Macro-gates (Block and Head logits) evaluate whether to retain or bypass entire MLP modules and attention heads. Concurrently, micro-gates (Neuron and Dimension logits) prune fine-grained structures within the surviving active structures. This dual-level formulation allows the network to autonomously carve out a compact, hardware-efficient sub-network during a single end-to-end training phase.}
\label{fig:hiap_overview}
\end{figure}

Our contributions are:
\begin{itemize}

\item We unify macro-level (heads, FFN blocks) and micro-level (intra-head dimensions, FFN neurons) structured pruning into a single differentiable phase, in which weights and architecture co-adapt without a separate search or retraining stage.


\item We decompose MACs into the specific computational costs of individual structures, creating a single objective that distributes sparsity across \textcolor{blue}{four} granularities without manual per-layer ratios or offline solvers. \textcolor{blue}{Because this cost model explicitly penalizes the overhead of an active structure, the network naturally learns to delete whole heads rather than keep them open at small width.}

\item We show that the allocation HiAP discovers depends on model scale. At an identical penalty ratio, DeiT-Base retains every FFN block and compresses attention, while DeiT-Small closes whole FFN blocks and keeps proportionally more heads. Methods that fix per-layer ratios cannot do this, since for them the allocation is an input rather than a result.

\item On ImageNet-1K, HiAP \textcolor{blue}{is competitive with} pruning pipelines of substantially greater complexity at comparable compute. \textcolor{blue}{On DeiT-Base it is more accurate than the other methods that also use distillation while using less compute, and without distillation it matches the accuracy drop of VTP and WDPruning with less compute.} Because the discovered sub-networks are physically extracted rather than masked, they realize more than 90\% of their theoretical MAC reduction as measured throughput \textcolor{blue}{at batch 256} on an A100, against 62--78\% for the strongest competing method.
\end{itemize}

\section{Related Work}
\begin{table*}[t]
\centering
\caption{\textbf{Comparison of recent structured pruning frameworks for ViTs.} Existing methods target distinct subsets of the architecture, such as focusing on depth or a mix of attention heads and internal dimensions. HiAP unifies \textcolor{blue}{four} architectural granularities to simultaneously optimize memory bandwidth (blocks/heads) and compute FLOPs (micro-structures).}
\label{tab:comparison}
\resizebox{\textwidth}{!}{%
\begin{tabular}{l l l l}
\toprule
\textbf{Method} & \textbf{Macro (Depth \& Heads)} & \textbf{Micro (Intra-Head \& FFN)} & \textbf{Search \& Budget Enforcement} \\
\midrule
ViT-Slim \citep{chavan2022vision} & None & Intra-Head Dims, FFN Neurons & $\ell_1$ Sparsity + Rank Thresholding \\
SAViT \citep{zheng2022savit} & Attention Heads & FFN Neurons, Embedding & Taylor Joint Optimization + EA \\
GOHSP \citep{yin2023gohsp} & Attention Heads & Intra-Head Dims, FFN Neurons & Graph-based Ranking + Optimization \\
NViT \citep{yang2023global} & Attention Heads & Intra-Head Dims, FFN Neurons & Latency-Aware Taylor Ranking \\
UPDP \citep{liu2024updp} & FFN Blocks & None & Genetic Algorithm \\
S$^2$ViTE \citep{chen2021chasing} & Attention Heads & FFN Neurons & Dynamic Sparse Training \\
WDPruning \citep{yu2022width} & Transformer Blocks, Heads & Intra-Head Dims, FFN Neurons & Saliency score \\
\textcolor{blue}{MDP \citep{sun2025mdp}} & \textcolor{blue}{Transformer Blocks, Heads} & \textcolor{blue}{Intra-Head Dims, FFN Neurons, Embedding} & \textcolor{blue}{Latency-Constrained MINLP Solver} \\
\rowcolor{gray!15} \textbf{HiAP (Ours)} & \textbf{FFN Blocks, Heads} & \textbf{Intra-Head Dims, FFN Neurons} & \textbf{End-to-End Auto-Penalty} \\
\bottomrule
\end{tabular}%
}
\end{table*}

\textbf{Structured Transformer Pruning.} Early structural pruning techniques established foundational heuristics like magnitude thresholding \citep{han2015deep}, second-order approximations \citep{hassibi1992second}, and Taylor-based scoring \citep{molchanov2019importance}. As these approaches transitioned to ViT, early studies showed that entire attention heads could be removed with little impact \citep{michel2019sixteen}. Recent structured pruning methods have evolved to target specific architectural granularities, with the majority focusing on compressing the network's width \citep{zhu2021vision, yu2022unified}. For instance, ViT-Slim \citep{chavan2022vision} searches for sparsity within intra-head dimensions and FFN channels. Other methods expand the search space to include entire attention heads. GOHSP \citep{yin2023gohsp} and X-Pruner \citep{yu2023x} prunes heads and intra-head/FFN columns using graph-based ranking \citep{Fang_2023_CVPR}. NViT \citep{yang2023global} and SAViT \citep{zheng2022savit} perform global structured pruning across heads and internal dimensions using Hessian-aware Taylor scoring and collaborative optimization. Recent efforts even strive to maintain isomorphic structures while reducing the transformer width \citep{fang2024isomorphic}.

Because the true latency bottleneck on modern hardware is often High-Bandwidth Memory (HBM) access rather than pure computation \citep{dao2022flashattention}, preserving all structural layers forces the hardware to incur significant memory-bound overhead of traversing every block and materializing every $N \times N$ attention map. To alleviate this, early literature explored shallowing deep networks \citep{8485719}, discrimination-based block pruning \citep{wang2019dbp}, interpretable layer pruning \citep{tang2023sr} and fusible residual pruning \citep{xu2020layer}. For ViT, macro-level methods like UPDP \citep{liu2024updp} and Layerdrop \citep{fan2019reducing} focus on dropping entire transformer blocks bypassing attention and FFN layers all together. \textcolor{blue}{Multi-granular methods prune several of these together.} \textcolor{blue}{MDP \citep{sun2025mdp} removes ViT blocks, heads and channels jointly by solving a mixed-integer program offline under a latency budget.}

In parallel, a separate line of work focuses on dynamic token pruning (reducing sequence length). Methods such as DynamicViT \citep{rao2021dynamicvit}, EViT \citep{liang2022not}, SPViT \citep{kong2022spvit}, A-ViT \citep{yin2022vit}, Peeling an Onion \citep{kong2023peeling} and ToMe \citep{bolya2022token} drop or merge tokens to reduce complexity. Our work strictly outputs a static, hardware-friendly sub-network, and it is entirely complementary to these input-dependent token-reduction techniques.

\textbf{Learnable Gates for Pruning.} Learnable gates or masks provide an alternative to fixed heuristics, allowing models to learn sparsity organically. Earlier work explored conditional computation \citep{bengio2015conditional} and sparse CNNs with continuous differentiable relaxations \citep{louizos2017learning}. In transformers, SViTE \citep{chen2021chasing} incorporated gating mechanisms, while X-Pruner \citep{yu2023x} trained masks to reflect class-wise importance. \textcolor{blue}{Closest to ours is CoFi \citep{xia2022structured}, which prunes layers, heads and widths of language models with nested $L_0$ masks, a target on parameter count and a separate fine-tuning stage.} \textcolor{blue}{HiAP uses a similar two-level hierarchy for ViTs, but associates each gate with its MAC cost, including the attention map, sets no target, and needs no fine-tuning after the run.}

\textbf{Budget-Aware Regularization.} For pruned models to be practically deployable, they must meet hardware constraints. Several works integrate budget constraints into the loss function or search space. ProxylessNAS \citep{cai2018proxylessnas} adds latency models to NAS, while AutoSlim \citep{yu2019autoslim} regularized channel counts. ViT pruning methods have adopted similar strategies. SPViT \citep{kong2022spvit} introduces a latency budget during token selection while NViT \citep{yang2023global} integrates latency-aware regularization into its structural pruning criteria. HiAP aligns with this idea and uses a loss function with the expected MACs, making the process both automatic and resource-aware.

\section{Method}
\label{sec:method}

HiAP solves a single budget-aware objective that allocates sparsity across four structural granularities. \textcolor{blue}{A linear, differentiable model of the expected compute is minimized together with the task loss, and Gumbel-Sigmoid gates make the keep/prune decisions differentiable.}

\textbf{Notation.} We consider a ViT with $L$ layers, $H$ attention heads, head dimension $D_h$, and feed-forward network (FFN) hidden width $D_{\text{ffn}}$. We introduce binary gates to regulate these structures. Macro-gates, denoted as $g_{l,h} \in \{0,1\}$ and $b_l \in \{0,1\}$, control the presence of entire attention heads and FFN blocks, respectively. Micro-gates, denoted as $d_{l,h,j} \in \{0,1\}$ and $c_{l,k} \in \{0,1\}$, control the finer capacity within these active macro-structures (i.e., specific attention dimensions and FFN neurons). We refer to the complete set of gates as $\mathcal{G}$.

\subsection{Hierarchical Gating Mechanism}
\label{subsec:gating}

\textcolor{blue}{Each Transformer block has two levels of gates.} \textcolor{blue}{The network can drop a coarse structure entirely to save a lot of compute, or only narrow it to keep its capacity.}

\textbf{Macro-Level Pruning.} The macro-gates control coarse units like entire attention heads ($g_{l,h}$) and FFN blocks ($b_l$). When $g_{l,h} = 0$, the $h$-th attention head in layer $l$ is completely bypassed. When $b_l = 0$, the entire FFN block in layer $l$ is removed.
\begin{align}
\text{AttnOut}_{l,h}(X) &= g_{l,h} \cdot \text{Attention}(X W^Q_{l,h}, X W^K_{l,h}, X W^V_{l,h}), \\
\text{FFNOut}_{l}(X) &= b_{l} \cdot \text{FFN}(X).
\end{align}

\textbf{Micro-Level Pruning.} The micro-gates operate within active macro-structures, and they dynamically prune the internal matrix dimensions. For an active attention head, the gates $d_{l,h} \in \{0,1\}^{D_h}$ prune the value-path dimensions. For an active FFN block, the gates $c_{l} \in \{0,1\}^{D_{\text{ffn}}}$ prune the intermediate hidden neurons:
\begin{align}
\text{Head}'_{l,h}(X) &= g_{l,h} \left[ \text{softmax}\!\left(\frac{Q_{l,h} K_{l,h}^\top}{\sqrt{D_h}}\right) \left( V_{l,h} \odot d_{l,h} \right) \right], \label{eq:micro_attn} \\
\text{FFN}'_{l}(X) &= b_{l} \left[ \left( \phi(X W_{1,l}) \odot c_{l}\right) W_{2,l} \right], \label{eq:micro_ffn}
\end{align}
where $\odot$ denotes channel-wise broadcasting and $\phi$ is the non-linear activation (e.g., GELU).

During training, these binary gates are relaxed to continuous variables $\hat{z} \in (0,1)$ using the Gumbel-Sigmoid distribution, enabling end-to-end differentiable optimization. When we export the pruned model, Equation~\ref{eq:micro_ffn} corresponds to the physical removal of specific columns from $W_{1,l}$ and their matching rows from $W_{2,l}$.

\subsection{Differentiable Cost Modeling}
\label{subsec:cost_model}

To guide the architecture search, we formulate an analytical, differentiable accounting of the network's Multiply-Accumulate operations (MACs). Let $N$ be the sequence length and $D$ be the embedding dimension. The computational cost of the network is decomposed into a static overhead $C^{\text{const}}$ (e.g., patch embedding, layer normalization) and the dynamic costs governed by our gates.

We isolate the marginal cost into three specific constants. \textcolor{blue}{They count two operations per multiply-accumulate, as FLOPs; every budget we report is in MACs, half these values:}
\begin{itemize}
    \item $C_1 = 2 N D (3 D_h) + 2 N^2 D_h$: The macro-overhead for computing the dense $Q, K, V$ projections and the attention map $Q K^\top$ for a single head.
    \item $C_2 = 2 N D + 2 N^2$: The micro-cost of computing the attention output for a single surviving value dimension $j$.
    \item $C_3 = 4 N D$: The micro-cost of computing a single surviving intermediate neuron in the FFN block.
\end{itemize}

Because the FFN macro-gate $b_l$ strictly acts as an enabler for its internal neurons $c_{l,k}$, the FFN block possesses no empty structural overhead. \textcolor{blue}{We leave out the static cost of unpruned layers (patch embedding, layer normalization, classification head), since its gradient with respect to the gates is zero.}

Therefore, the total prunable computational cost of the search space $\mathcal{G}$ is linearly decomposed as:
\begin{equation}
\begin{aligned}
\mathbb{E}[C(\mathcal{G})] ={}&
\sum_{l=1}^{L} \sum_{h=1}^{H} \Big(
C_1 \cdot \mathbb{E}[g_{l,h}]
+ C_2 \sum_{j=1}^{D_h} \mathbb{E}[g_{l,h} \cdot d_{l,h,j}]
\Big) \\
&{}+ \sum_{l=1}^{L} \sum_{k=1}^{D_{\text{ffn}}}
C_3 \cdot \mathbb{E}[b_{l} \cdot c_{l,k}].
\end{aligned}
\label{eq:total_flops}
\end{equation}

This linear decomposition is crucial - it allows the optimization process to cleanly attribute hardware penalties to individual structures, explicitly penalizing the network for keeping empty attention heads open, while allowing FFN blocks to scale costs purely by their active neuron count. \textcolor{blue}{The cost model is per structure, while extraction pads the heads of a layer to the widest surviving one, so the two differ slightly in layers that lose value dimensions.} \textcolor{blue}{Surviving heads in the reported models keep nearly all of their value dimensions (Section~\ref{subsec:structure_analysis}), so the difference is small.}

\subsection{Training with Gumbel-Sigmoid}
\label{subsec:objective}

To train the binary gates, we use the Gumbel-Sigmoid relaxation. Each gate is parameterized by a learnable gate logit $\alpha$. It learns the structural importance of its corresponding architectural component. During the forward pass, we sample a continuous gate value $\hat{z} \in (0, 1)$ by adding logistic noise $\epsilon$ and applying a sigmoid with temperature $\tau$:
\begin{equation}
\hat{z} = \sigma\!\left(\frac{\alpha + \epsilon}{\tau}\right). \label{eq:gumbel}
\end{equation}

The total optimization objective during training is a combination of the primary task loss, cost-aware structural regularization, and a set of architectural feasibility penalties:
\begin{equation}
\mathcal{L}_{\text{total}} = \mathcal{L}_{\text{task}} + \lambda_{\text{macro}} \mathcal{L}_{\text{macro}} + \lambda_{\text{micro}} \mathcal{L}_{\text{micro}} + \mathcal{L}_{\text{feasibility}}.
\end{equation}

For the task loss, we use a combination of standard cross-entropy and Knowledge Distillation (KD). A pre-trained, dense teacher model provides soft targets, which are crucial for maintaining the representations of the student network as its capacity is dynamically pruned.

Rather than using a global, squared-error budget target, we directly penalize the expected macro and micro computational costs derived in Equation~\ref{eq:total_flops}\textcolor{blue}{, each normalized by its dense value}. We decouple these into $\mathcal{L}_{\text{macro}}$ (penalizing $C_1$) and $\mathcal{L}_{\text{micro}}$ (penalizing $C_2$ and $C_3$), controlled by independent hyperparameters $\lambda_{\text{macro}}$ and $\lambda_{\text{micro}}$. This decoupling allows us to explicitly manage the trade-off between coarse and fine-grained sparsity.

A common failure mode in differentiable architecture search is structural collapse, where the network greedily prunes entire layers to minimize the cost penalty before the weights can adapt. To prevent this, we introduce a feasibility penalty $\mathcal{L}_{\text{feasibility}}$ that enforces minimum retention quotas using a ReLU threshold:
\begin{equation}
\mathcal{L}_{\text{feasibility}} = \beta_{\text{head}} \mathcal{L}_{f,\text{head}} + \beta_{\text{dim}} \mathcal{L}_{f,\text{dim}} + \beta_{\text{ffn}} \mathcal{L}_{f,\text{ffn}}.
\end{equation}
For example, $\mathcal{L}_{f,\text{head}} = \sum_{l} \text{ReLU}(k_{\text{min}} - \sum_{h} g_{l,h})^2$ penalizes layers whose active head count falls below a minimum $k_{\text{min}}$. Analogous terms apply a minimum-ratio threshold to the surviving attention dimensions ($\gamma_{\text{attn}}$) within each active head and to the FFN neurons ($\gamma_{\text{ffn}}$) within each active block, preventing structural collapse and keeping gradients well conditioned.

\subsection{Single-Phase End-to-End Discovery}
\label{subsec:single_phase}

A significant limitation of prior differentiable architecture search methods is their reliance on a two-phase pipeline: an initial search phase to discover the architecture mask, followed by a mandatory, isolated fine-tuning phase to recover the degraded weights.

HiAP avoids this split by unifying search and training into a single, continuous phase. The dense ViT is trained end-to-end alongside the gate logits using $\mathcal{L}_{\text{total}}$. Throughout the training process, the Gumbel-Sigmoid temperature $\tau$ is gradually annealed from an initial value $\tau_0$ to a near-zero minimum $\tau_{\text{min}}$ (as illustrated in Figure~\ref{fig:gumbel_annealing}). The training process is summarized in Alg. \ref{alg:hiap}.

\begin{figure}[h]
\centering
\includegraphics[width=\linewidth]{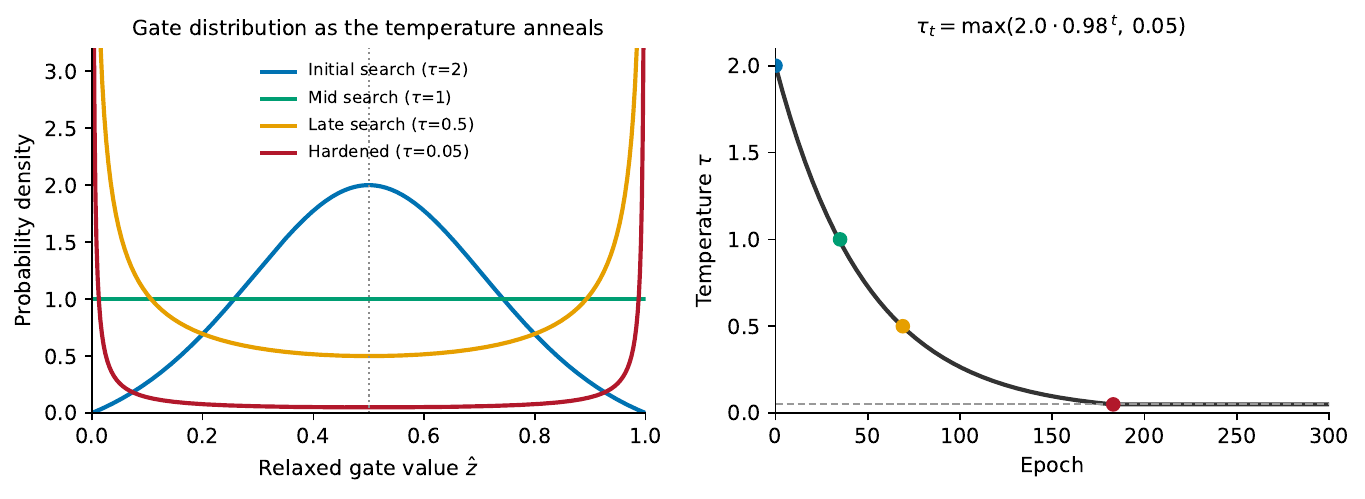} 
\caption{Gumbel-Sigmoid annealing under the schedule used in our experiments. Left: the distribution of the relaxed gate value at four temperatures. Early in training the gates are soft and act as a continuous regularizer; as $\tau$ falls the mass moves to 0 and 1 and the network hardens into a discrete sub-architecture \textcolor{blue}{with no threshold to tune}. The two low-temperature curves continue above the plotted range near 0 and 1. Right: the temperature schedule, with each temperature marked at the epoch it is reached.}
\label{fig:gumbel_annealing}
\end{figure}

\begin{algorithm}[t]
\caption{HiAP single-phase auto-pruning. $\mathcal{G}$ is the gate set with logits $\alpha$ (head $g$, block $b$, value dimension $d$, neuron $c$); $\theta$ the network weights; $\mathcal{T}$ a pretrained teacher; $\tau$ the Gumbel-Sigmoid temperature annealed $\tau_0\!\to\!\tau_{\min}$; $\lambda_{\text{macro}}, \lambda_{\text{micro}}$ the budget weights; $\eta_\theta, \eta_\alpha$ the learning rates.}
\label{alg:hiap}
\begin{algorithmic}[1]
\State Initialize $\theta$ from a dense pretrained ViT; open all gates ($\alpha \gets \alpha_0 > 0$); precompute $C_1, C_2, C_3$ \Comment{Eq.~\ref{eq:total_flops}}
\For{epoch $e = 1, 2, \dots$}
    \State anneal temperature $\tau$ for epoch $e$
    \For{minibatch $(x, y)$}
        \State sample relaxed gates $\hat z = \sigma\!\big((\alpha + \epsilon)/\tau\big)$ and
        \Statex \hspace*{3em} run the forward pass \Comment{Eqs.~\ref{eq:gumbel},~\ref{eq:micro_attn},~\ref{eq:micro_ffn}}
        \State $\mathcal{L}_{\text{task}} \gets \mathrm{KD}(\text{student}, \mathcal{T}, y)$
        \State $\mathcal{L} \gets \mathcal{L}_{\text{task}} + \lambda_{\text{macro}}\mathcal{L}_{\text{macro}} + \lambda_{\text{micro}}\mathcal{L}_{\text{micro}} + \mathcal{L}_{\text{feas}}$
        \State update $\theta$ and $\alpha$ by gradient descent on $\mathcal{L}$
    \EndFor
\EndFor
\State \textcolor{blue}{harden gates: keep a structure if $\sigma(\alpha) > 0.5$}
\State delete pruned heads, blocks, head dimensions and neurons
\State \Return compact sub-network
\end{algorithmic}
\end{algorithm}

\textbf{Dynamic Co-Adaptation.} \textcolor{blue}{Early in training (high $\tau$), the gates act like stochastic dropout and push the surviving weights to learn robust, distributed features.} As $\tau$ decays the gate distributions sharpen toward binary decisions. \textcolor{blue}{Since the weights adapt to this slowly hardening topology, the network reaches its final sparse state without the sudden accuracy drop of one-shot structural pruning.}

\textbf{Physical Sub-network Extraction.} In prior works, sub-networks are often left with "soft" masks or \textcolor{blue}{need sparse kernels} to realize speedups. \textcolor{blue}{HiAP instead physically extracts the discovered architecture at the end of training.} \textcolor{blue}{Gates are hardened at $\sigma(\alpha) > 0.5$, pruned heads and FFN blocks are deleted, and surviving value dimensions and FFN neurons are cut out of the weight matrices.} \textcolor{blue}{The result is an ordinary dense ViT with narrower layers that runs on stock kernels with no further fine-tuning.}

\section{Theoretical Remarks}

\begin{lemma}[Expressivity: strict superset]
Let $\mathcal{A}_{\text{head}}$ be the set of architectures reachable by head/block (macro) gates only, and let $\mathcal{A}_{\text{hiap}}$ be those additionally using micro-gates over attention dimensions or FFN neurons. If any layer has $D_h > 1$ or $D_{\text{ffn}} > 1$, then $\mathcal{A}_{\text{head}} \subsetneq \mathcal{A}_{\text{hiap}}$.
\end{lemma}

\begin{proof}[Proof sketch]
Fix a layer/head with $D_h > 1$ (the FFN case is analogous). Pick an architecture that keeps a head ($g_{l,h} = 1$) but zeros a strict subset of its dimensions with $d_{l,h} \in \{0,1\}^{D_h}, \; d_{l,h} \notin \{\mathbf{0}, \mathbf{1}\}$. This cannot be realized by macro gates alone (which toggle whole heads), hence $\mathcal{A}_{\text{head}} \subsetneq \mathcal{A}_{\text{hiap}}$.
\end{proof}

\begin{proposition}[Budget linearity]
Under the accounting convention above, there exist nonnegative weights $\{w_i\}$ for each prunable unit such that the expected prunable cost decomposes linearly as $\mathbb{E}[C] = \sum_i w_i \, \mathbb{E}[z_i]$, where $z_i \in \{ g_{l,h}, \; g_{l,h} d_{l,h,j}, \; b_l c_{l,k} \}$. No independence assumptions are required.
\end{proposition}

\begin{proof}[Proof sketch]
Define $w_i$ as the marginal MACs saved by removing unit $i$ (see Appendix). By linearity of expectation, $\mathbb{E}[C] = \sum_i w_i \, \mathbb{E}[z_i]$ holds regardless of correlations among gates, since each term is a linear functional of the corresponding binary variable.
\end{proof}

\noindent \textbf{Annealing and hardening.} As $\tau \to 0$, the Gumbel-Sigmoid entropy decreases and samples concentrate on threshold decisions; a soft-to-hard schedule stabilizes learning, aligning with the single-phase discovery and direct hardening. \textcolor{blue}{Every accuracy and MAC count we report is measured on the hardened network, which is exactly the architecture that is extracted.}

\section{Experiments}

We evaluate HiAP on Imagenet-1k. Our experiments demonstrate that HiAP can compress standard Vision Transformers to \textcolor{blue}{a range of} computational budgets in a single training phase, producing physically extractable subnetworks without the need for secondary fine-tuning.

\subsection{Experimental Setup}

\textbf{Datasets and Baselines.} We evaluate on ImageNet-1K, applying HiAP to DeiT-Base (17.6G MACs) and DeiT-Small (4.6G MACs).


\textbf{Implementation Details.} HiAP finds and trains the sub-network in one phase. Both models start from pre-trained DeiT weights and train for 300 epochs with a global batch size of 256, \textcolor{blue}{distilling from a dense, pre-trained RegNetY-160}. \textcolor{blue}{The two DeiT-Base models without distillation use the same recipe without the teacher.} \textcolor{blue}{The penalty weights and all other settings for every reported model are listed in Appendix~\ref{app:hparams}.} \textcolor{blue}{All results come from a single run per configuration.} \textcolor{blue}{We did not measure variance across seeds, so we do not treat differences of a few tenths of a point between methods as significant.}

\subsection{ImageNet-1K}

Tables~\ref{tab:base_results} and~\ref{tab:main_results} compare HiAP against published ViT pruning methods at matched compute. \textcolor{blue}{DeiT is released in two versions, with and without distillation, and every listed method prunes the version without it.} \textcolor{blue}{NViT \citep{yang2023global}, MDP \citep{sun2025mdp} and Isomorphic Pruning \citep{fang2024isomorphic} start from the distilled DeiT-Base (83.3\%), so we do not list them; MDP also reports speedup rather than MACs.} \textcolor{blue}{Some methods, ours included, then train the pruned model against a teacher, and these are marked $^*$.} \textcolor{blue}{Where a method reports both a structure-only result and one that also removes patches or tokens, we list the structure-only one.} \textcolor{blue}{On DeiT-Base we also report HiAP trained without a teacher.}

\begin{table}[h]
\centering
\caption{Comparison of HiAP against ViT pruning methods on DeiT-Base. Competitor deltas are computed against the dense baseline reported in each work; ours against our measured baseline of 81.71\% at 17.6G MACs. \textcolor{blue}{$^*$Trained with knowledge distillation: LPViT and HiAP from a pre-trained RegNetY-160, UVC from the unpruned model.} \textcolor{blue}{$^\dagger$The patch pruning reimplementation of SCOP by \citet{tang2022patch}, as reported by \citet{yu2022width}.} \textcolor{blue}{SSP is the structured pruning baseline reported by \citet{chen2021chasing}.}}
\label{tab:base_results}
\resizebox{0.7\textwidth}{!}{%
\begin{tabular}{lcccc}
\toprule
\textbf{Method} & \textbf{MACs (G)} & \textbf{Baseline} & \textbf{Top-1 (\%)} & \textbf{vs.\ dense} \\
\midrule
\textcolor{blue}{VTP (20\%) \citep{zhu2021vision}} & 13.8 & 81.80 & 81.30 & $-0.50$ \\
S$^2$ViTE \citep{chen2021chasing} & 11.8 & 81.80 & 82.22 & $+0.42$ \\
\textcolor{blue}{SSP \citep{chen2021chasing}} & 11.8 & 81.80 & 80.08 & $-1.72$ \\
\midrule
\textcolor{blue}{SCOP$^\dagger$ \citep{tang2020scop}} & 10.2 & 81.80 & 79.70 & $-2.10$ \\
\textcolor{blue}{VTP (40\%) \citep{zhu2021vision}} & 10.0 & 81.80 & 80.70 & $-1.10$ \\
WDPruning \citep{yu2022width} & 9.9 & 81.80 & 80.76 & $-1.04$ \\
\textcolor{blue}{HiAP (Ours)} & \textcolor{blue}{9.1} & \textcolor{blue}{81.71} & \textcolor{blue}{80.62} & \textcolor{blue}{$-1.09$} \\
\textcolor{blue}{\textbf{HiAP (Ours)}$^*$} & \textbf{10.1} & \textbf{81.71} & \textbf{81.13} & \textbf{$-0.58$} \\
\midrule
SAViT-50 \citep{zheng2022savit} & 8.8 & 81.84 & 82.54 & $+0.70$ \\
\textcolor{blue}{LPViT(1)$^*$ \citep{xu2024lpvit}} & 8.8 & 81.80 & 80.81 & $-0.99$ \\
X-Pruner \citep{yu2023x} & 8.5 & 81.80 & 81.02 & $-0.78$ \\
\textcolor{blue}{UVC$^*$ \citep{yu2022unified}} & 8.0 & 81.80 & 80.57 & $-1.23$ \\
\textcolor{blue}{LPViT(2)$^*$ \citep{xu2024lpvit}} & 7.9 & 81.80 & 80.55 & $-1.25$ \\
\textcolor{blue}{HiAP (Ours)} & \textcolor{blue}{7.8} & \textcolor{blue}{81.71} & \textcolor{blue}{80.40} & \textcolor{blue}{$-1.31$} \\
\textcolor{blue}{\textbf{HiAP (Ours)}$^*$} & \textbf{7.4} & \textbf{81.71} & \textbf{80.88} & \textbf{$-0.83$} \\
\bottomrule
\end{tabular}%
}
\end{table}

\begin{table}[h]
\centering
\caption{Comparison of HiAP against ViT pruning methods on DeiT-Small, with a dense baseline of 4.6G MACs. \textcolor{blue}{$^*$Trained with knowledge distillation: HiAP from a pre-trained RegNetY-160, UVC from the unpruned model.} \textcolor{blue}{ViT-Slim rows use its dimension search only, without patch removal.}}
\label{tab:main_results}
\resizebox{0.7\columnwidth}{!}{%
\begin{tabular}{lcccc}
\toprule
\textbf{Method} & \textbf{MACs (G)} & \textbf{Baseline} & \textbf{Top-1 (\%)} & \textbf{vs.\ dense} \\
\midrule
WDPruning \citep{yu2022width} & 3.1 & 79.80 & 78.55 & $-1.25$ \\
S$^2$ViTE \citep{chen2021chasing} & 3.1 & \textcolor{blue}{79.90} & 79.22 & \textcolor{blue}{$-0.68$} \\
\textcolor{blue}{ViT-Slim \citep{chavan2022vision}} & \textcolor{blue}{3.3} & \textcolor{blue}{79.90} & \textcolor{blue}{80.03} & \textcolor{blue}{$+0.13$} \\
GOHSP \citep{yin2023gohsp} & 3.0 & 79.90 & 79.98 & $+0.08$ \\
\textcolor{blue}{\textbf{HiAP (Ours)}$^*$} & \textbf{3.1} & \textbf{79.75} & \textbf{79.33} & \textbf{$-0.42$} \\
\midrule
WDPruning \citep{yu2022width} & 2.6 & 79.80 & 78.38 & $-1.42$ \\
\textcolor{blue}{ViT-Slim \citep{chavan2022vision}} & \textcolor{blue}{2.8} & \textcolor{blue}{79.90} & \textcolor{blue}{79.20} & \textcolor{blue}{$-0.70$} \\
GOHSP \citep{yin2023gohsp} & 2.8 & 79.90 & 79.86 & $-0.04$ \\
\textcolor{blue}{UVC$^*$ \citep{yu2022unified}} & \textcolor{blue}{2.7} & \textcolor{blue}{79.80} & \textcolor{blue}{79.44} & \textcolor{blue}{$-0.36$} \\
\textcolor{blue}{\textbf{HiAP (Ours)}$^*$} & \textbf{2.9} & \textbf{79.75} & \textbf{78.64} & \textbf{$-1.11$} \\
\bottomrule
\end{tabular}%
}
\end{table}



\textbf{DeiT-Base.} At 10.1G MACs HiAP reaches 81.13\%, 0.58 points below our dense baseline of 81.71\%. \textcolor{blue}{In the same band it scores 0.43 points above VTP and 0.37 above WDPruning at about the same compute, and 1.43 above SCOP, though none of them uses distillation.} Tightening the budget to 7.4G gives 80.88\%. \textcolor{blue}{That point uses less compute than both UVC (8\% less) and LPViT (6\% less) while matching or slightly exceeding their accuracy, by 0.31 and 0.33 points.} \textcolor{blue}{Both of these distill as well, so this is the fairest comparison in the table.} It also lands within 0.42 points of VTP's 13.8G result using 46\% less compute. \textcolor{blue}{Without distillation HiAP reaches 80.62\% at 9.1G and 80.40\% at 7.8G, about half a point below the distilled models, and still matches the accuracy drop of VTP and WDPruning (1.09 points against 1.10 and 1.04) with 8 to 9\% less compute.}

\textbf{DeiT-Small.} On the smaller model HiAP reaches 79.33\% at 3.1G, ahead of WDPruning by 0.78 points and \textcolor{blue}{level with} S$^2$ViTE at identical compute, and 78.64\% at 2.9G. \textcolor{blue}{ViT-Slim, GOHSP and UVC stay ahead at this scale, and UVC does so with less compute.} \textcolor{blue}{All three} search a supernet or solve for a per-layer allocation, which is a materially different pipeline (Appendix C). The gap also reflects how much less redundancy the smaller model carries. HiAP removes two thirds of the attention heads from DeiT-Base for a 0.58 point drop, but only about half from DeiT-Small for a comparable one. The method appears to scale better when there is more to find.


\textcolor{blue}{Appendix C compares the training pipeline of each method.} \textcolor{blue}{Of the four approaches with shorter pipelines than ours (WDPruning, X-Pruner, UVC and VTP), only UVC exceeds our distilled results at matched compute, and only on DeiT-Small.} \textcolor{blue}{The others that exceed them, SAViT, ViT-Slim and GOHSP, each add a separate importance-estimation, search or ranking stage before the final retraining.} \textcolor{blue}{HiAP occupies a different point: a single training phase with no separate importance estimate, solver, ranking step or fine-tuning pass.} \textcolor{blue}{The two penalty weights are the only settings we change between runs, and everything else is fixed for each model (Appendix~\ref{app:hparams}).} \textcolor{blue}{Because the penalties set how much is removed rather than a hard budget, we trained six penalty settings on DeiT-Base to cover the compute bands of the published baselines.} The accuracy this achieves is competitive within the band rather than the highest in it, and the resulting sub-network is dense by construction, which Section~\ref{subsec:inference} shows is where the practical difference lies.

\subsection{Hierarchical Search Dynamics and Structure Analysis}
\label{subsec:structure_analysis}


\textcolor{blue}{Because the final architecture is read off after training rather than set before it, we can ask what HiAP does with a compute budget.}

\textbf{Macro first, then micro, then weights.} We observe that the head and block gates move earliest, since closing a coarse structure gives the largest immediate drop in $\mathcal{L}_{\text{macro}}$, and they settle within roughly the first quarter of training. FFN widths keep adjusting well past that point. \textcolor{blue}{After that the topology stops changing and only the weights are updated.} \textcolor{blue}{Accuracy follows this: it typically} falls by around three points while the architecture is being decided, then climbs back over the long tail to finish at the numbers reported in Table~\ref{tab:base_results}. The run never restarts and we never specify where the search ends and recovery begins; the schedule allocates that itself. Methods that separate the two must choose the split in advance and pay for a full retraining run afterwards (Appendix C).



\textbf{What gets removed depends on model size.} At an identical 2:1 penalty ratio, HiAP reaches different structural solutions on the two models. On DeiT-Base it keeps every FFN block for the whole run and takes its savings from attention, retaining 49 of 144 heads. On DeiT-Small it chooses to close whole FFN blocks including the last one in \textcolor{blue}{both reported models}, while keeping around half the heads. A head and an FFN neuron do not scale the same way with embedding dimension, so which one is the cheaper source of MACs changes as the model becomes bigger. Methods that fix per-layer ratios cannot do this, because for them the allocation is an input rather than a result. \textcolor{blue}{Solver and ranking methods such as NViT \citep{yang2023global}, GOHSP \citep{yin2023gohsp} and MDP \citep{sun2025mdp} can also reach allocations that depend on the model, but in a separate step rather than inside the training run.} \textcolor{blue}{We see this at one penalty ratio on two models, so we treat it as an observation.}



\textbf{Dropping whole heads is favoured over thinning.} Removing a whole head saves much more than making a head narrower because a head that stays open pays $C_1$ for its $Q$\textcolor{blue}{, $K$ and $V$} projections and its attention map regardless of how many value dimensions survive. On DeiT-Base that fixed part is 72\% of what a full head costs. The optimizer takes the larger saving by removing heads entirely and leaving the survivors at full width. \textcolor{blue}{In all four reported models, surviving heads keep all 64 value dimensions except a layer or two.} Raising $\lambda_{\text{micro}}$ relative to $\lambda_{\text{macro}}$ shifts this, and value dimensions then fall to roughly half width across the network (Sec.~\ref{subsec:ratios}), trading head count for head width at similar compute. The ratio decides which granularity does the work, not just how much is removed.

\textbf{The search is reversible.} Gate decisions are not final once made. For example, in the 1.5:1 configuration (Fig \ref{fig:evol_15_1}) an FFN block closes early in training and stays closed for around fifteen epochs before reopening and remaining open for the rest of the run. Because the gates stay differentiable throughout, a decision made while the weights are still adapting can be undone once they have moved. Methods that rank structures once and remove them have no equivalent mechanism, and an early misjudgment is permanent.

\subsection{Sensitivity to the Macro-Micro Penalty Ratio}
\label{subsec:ratios}

\begin{figure}[t]
    \centering
    \begin{subfigure}[t]{0.48\linewidth}\centering
        \includegraphics[width=\linewidth]{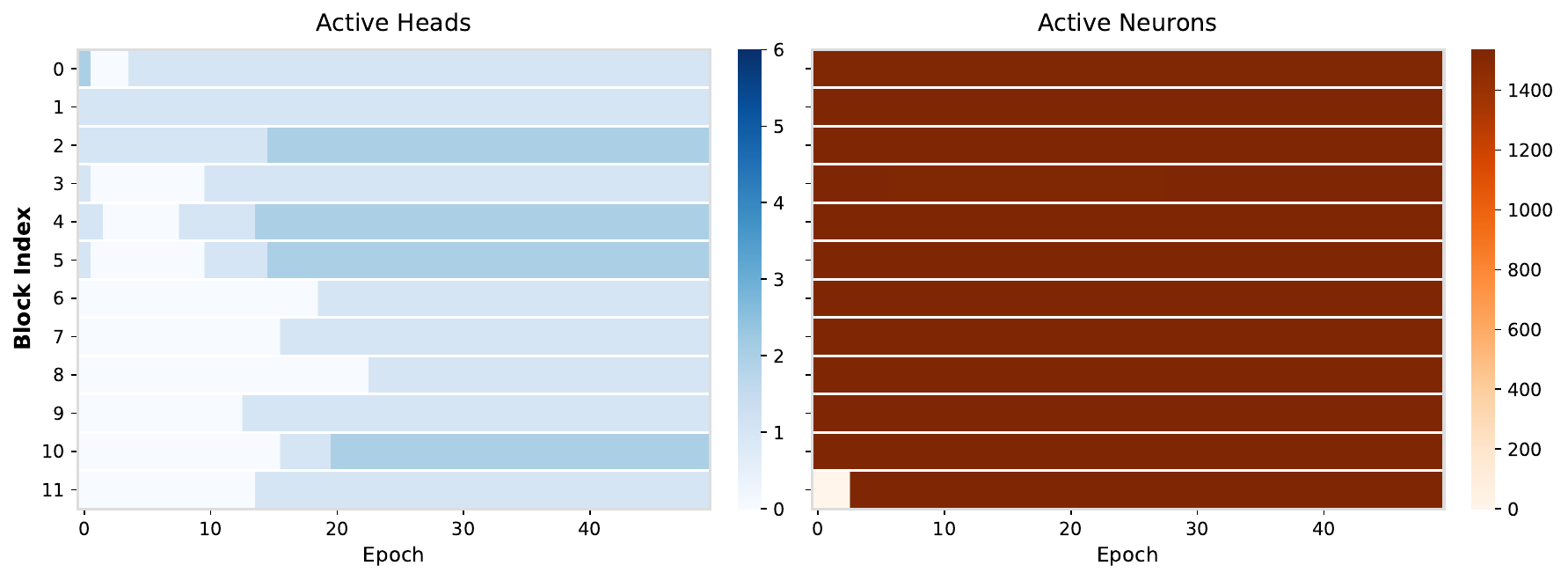}
        \caption{Macro-only: $\lambda_{\text{macro}}{=}1.5,\ \lambda_{\text{micro}}{=}0.0$}
        \label{fig:evol_macro}
    \end{subfigure}\hfill
    \begin{subfigure}[t]{0.48\linewidth}\centering
        \includegraphics[width=\linewidth]{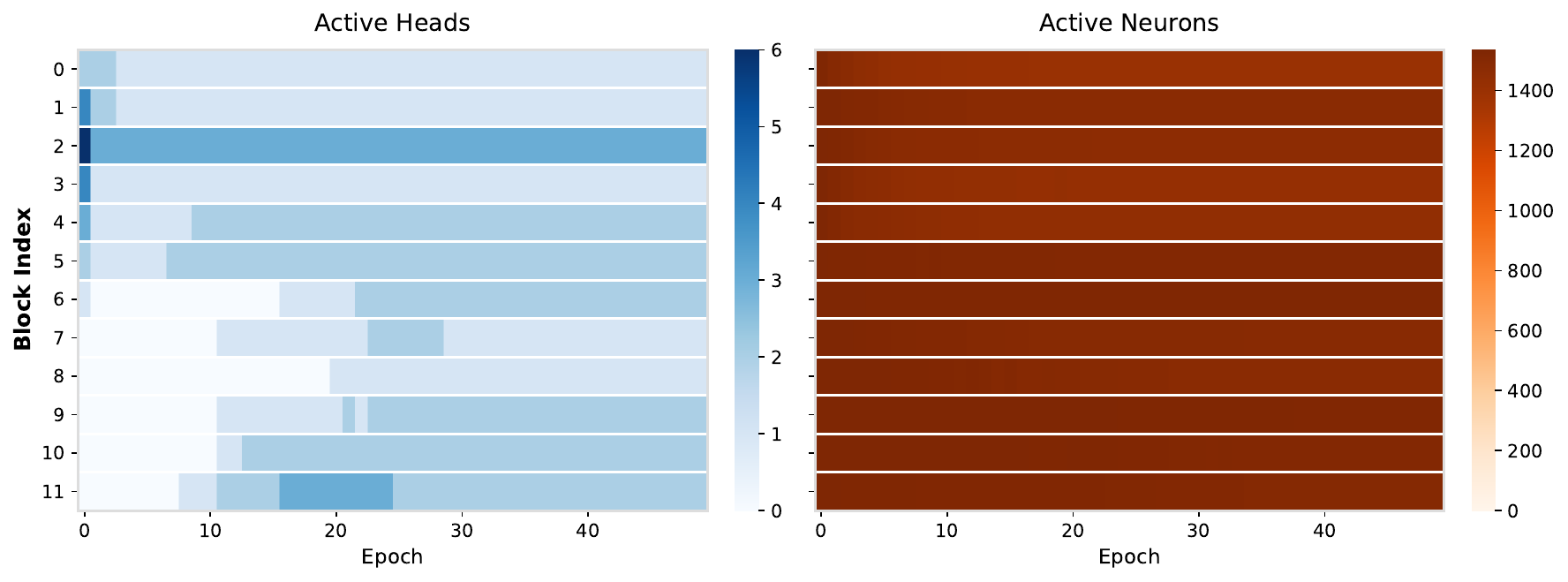}
        \caption{5:1: $\lambda_{\text{macro}}{=}1.0,\ \lambda_{\text{micro}}{=}0.2$}
        \label{fig:evol_5_1}
    \end{subfigure}

    \begin{subfigure}[t]{0.48\linewidth}\centering
        \includegraphics[width=\linewidth]{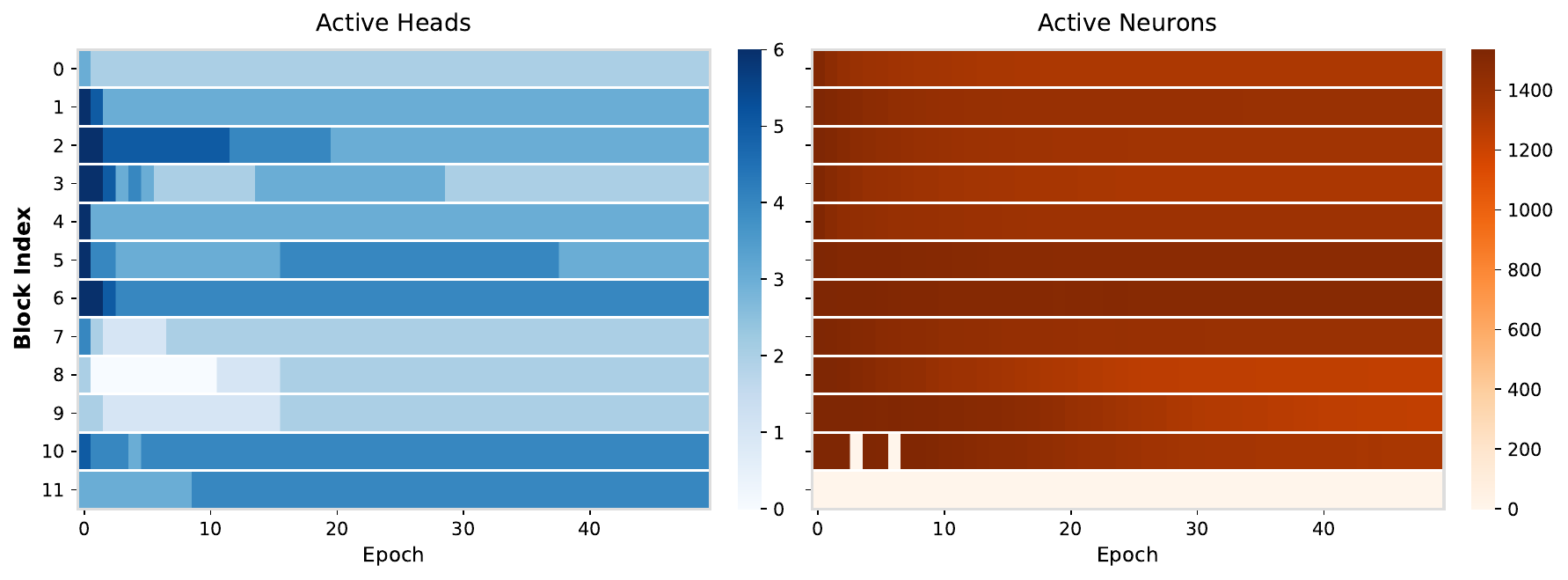}
        \caption{2:1: $\lambda_{\text{macro}}{=}0.9,\ \lambda_{\text{micro}}{=}0.45$}
        \label{fig:evol_2_1}
    \end{subfigure}\hfill
    \begin{subfigure}[t]{0.48\linewidth}\centering
        \includegraphics[width=\linewidth]{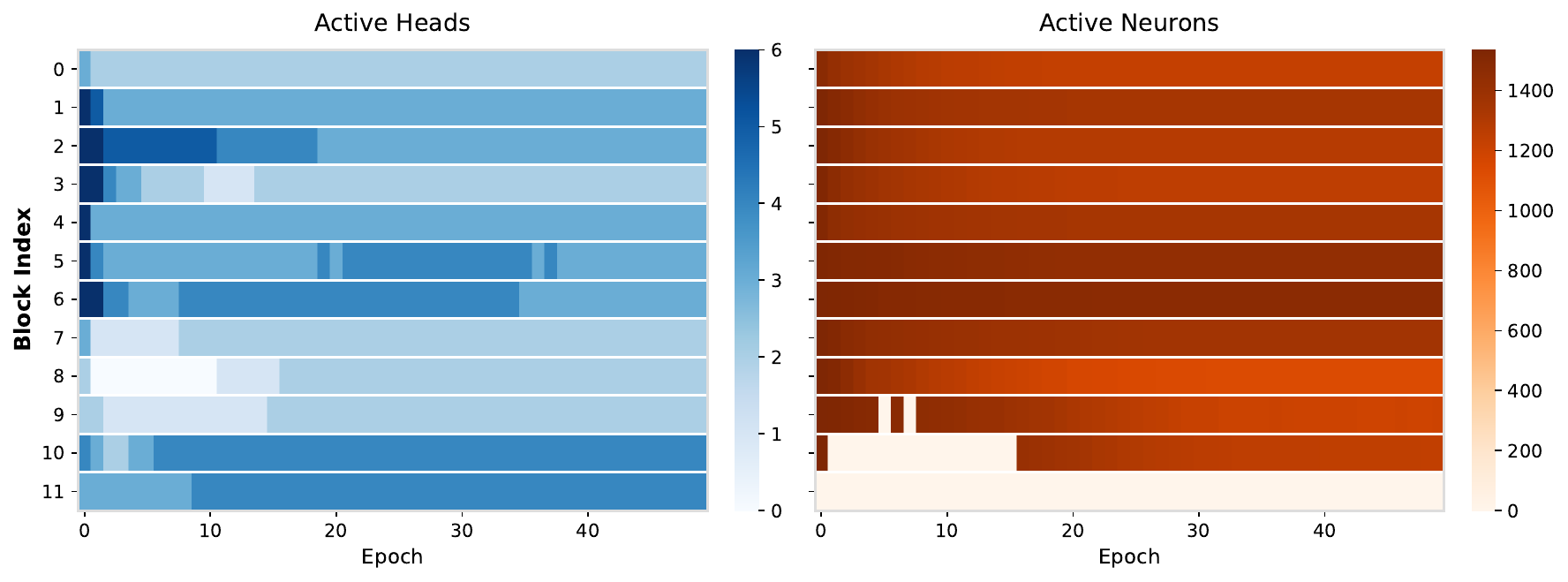}
        \caption{1.5:1: $\lambda_{\text{macro}}{=}1.2,\ \lambda_{\text{micro}}{=}0.8$}
        \label{fig:evol_15_1}
    \end{subfigure}

    \begin{subfigure}[t]{0.48\linewidth}\centering
        \includegraphics[width=\linewidth]{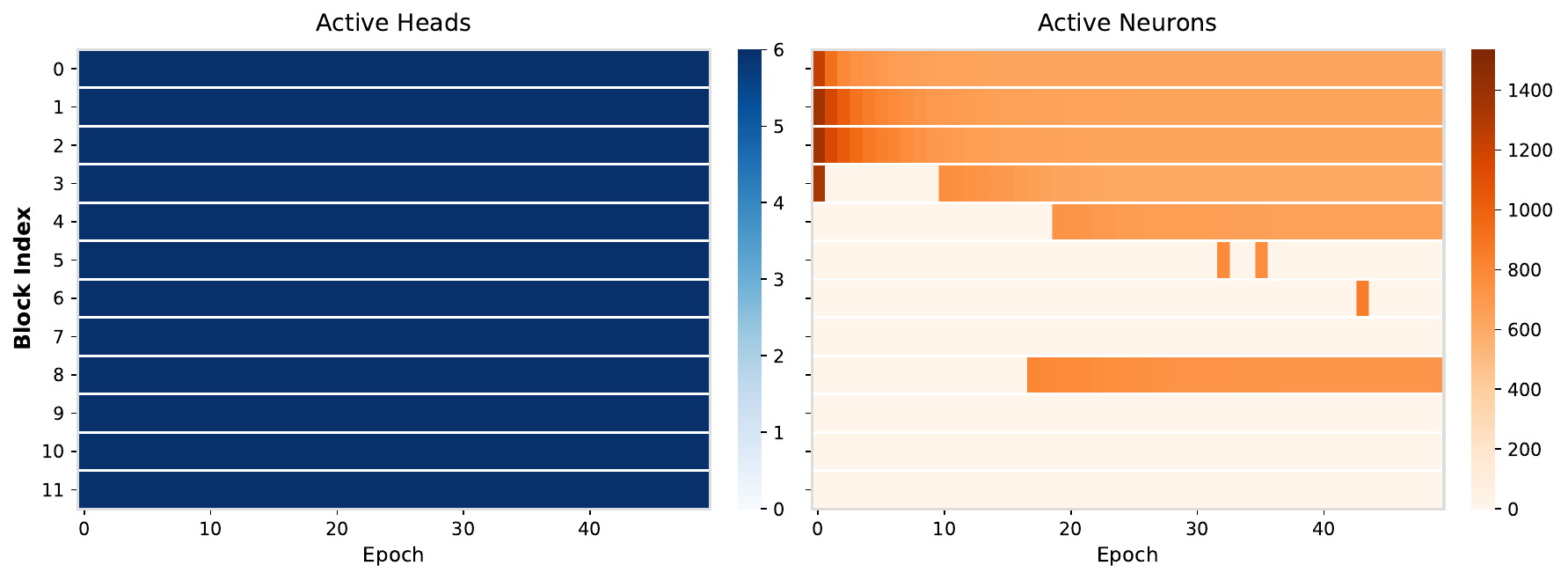}
        \caption{Micro-only: $\lambda_{\text{macro}}{=}0.0,\ \lambda_{\text{micro}}{=}1.5$}
        \label{fig:evol_micro}
    \end{subfigure}
    \caption{Architecture evolution under different macro:micro penalty ratios over the first 50 epochs (DeiT-Small, ImageNet). Each panel shows per-layer retention of attention heads and FFN neurons across training; darker is more retained, lighter more pruned. Configurations span macro-only (a) through micro-only (e).}
    \label{fig:lambda_evol}
\end{figure}

\textcolor{blue}{We train DeiT-Small for 50 epochs at a range of macro:micro ratios, from macro-only through 5:1, 2:1 and 1.5:1 to micro-only.} \textcolor{blue}{Figure~\ref{fig:lambda_evol} shows how the architecture evolves, and the Appendix gives accuracy against cost.} \textcolor{blue}{The overall penalty strength also differs between these settings, so the comparison is indicative.} \textcolor{blue}{The 2:1 ratio (Figure~\ref{fig:lambda_evol}(c)) gives the best balance for DeiT-Small, with a stable pruning order from macro to micro structures.}

\textcolor{blue}{With macro only (Figure~\ref{fig:lambda_evol}(a)), most heads are removed and several blocks are bypassed early on for a quick drop in compute.} \textcolor{blue}{The residual connections keep gradients flowing, so accuracy does not collapse, but the MLP blocks keep nearly all their neurons and the remaining budget is spent unevenly.}

\textcolor{blue}{With micro only (Figure~\ref{fig:lambda_evol}(e)), every head is kept, but many MLP blocks are emptied early in training, though some recover later.} \textcolor{blue}{The surviving MLP blocks keep only a fraction of their neurons.}

\textcolor{blue}{The 5:1 ratio (Figure~\ref{fig:lambda_evol}(b)) mostly drops heads and leaves the MLPs bloated, while 1.5:1 (Figure~\ref{fig:lambda_evol}(d)) gives a more balanced topology close to 2:1.}

In summary, these configurations form a loss-term ablation. Macro-only and micro-only each yield unbalanced allocations (bloated MLPs, or unstable indirect block elimination), whereas only the coupled penalty distributes sparsity cleanly across both granularities.

\subsection{Inference Speed of Extracted Sub-Networks}
\label{subsec:inference}

\begin{table}[h]
\centering
\caption{Measured inference speed of physically extracted HiAP sub-networks on a single NVIDIA A100 80GB, FP32, PyTorch 2.4. Latency is batch size 1, throughput is batch size 256, both averaged over 100 runs after 50 warm-up iterations. Baselines are the stock \texttt{timm} DeiT models, re-measured alongside every pruned model; the reported baseline is the mean, with all runs agreeing to within 0.3\%. Extracted widths are padded to the nearest multiple of 8 to ensure hardware kernel alignment, preserving bit-identical outputs.}
\label{tab:latency}
\resizebox{0.7\columnwidth}{!}{%
\begin{tabular}{lccc}
\toprule
\textbf{Model} & \textbf{MACs (G)} & \textbf{Latency, ms ($\times$)} & \textbf{Throughput, img/s ($\times$)} \\
\midrule
DeiT-Base (dense) & 17.6 & 4.69 (1.00) & 474.3 (1.00) \\
\textbf{HiAP (Ours)} & \textbf{10.1} & \textbf{3.30 (1.42)} & \textbf{793.8 (1.67)} \\
\textbf{HiAP (Ours)} & \textbf{7.4} & \textbf{2.79 (1.68)} & \textbf{1054.8 (2.22)} \\
\midrule
DeiT-Small (dense) & 4.6 & 3.04 (1.00) & 1656.7 (1.00) \\
\textbf{HiAP (Ours)} & \textbf{3.1} & \textbf{2.69 (1.13)} & \textbf{2248.1 (1.36)} \\
\textbf{HiAP (Ours)} & \textbf{2.9} & \textbf{2.59 (1.17)} & \textbf{2422.3 (1.46)} \\
\bottomrule
\end{tabular}%
}
\end{table}

\begin{table}[h]
\centering
\caption{How much of the theoretical compute reduction survives as measured speedup. Each method is compared against its own dense baseline on its own hardware, so the ratio is meaningful across papers even though absolute timings are not. SAViT is measured on V100, NViT on V100, HiAP on A100.}
\label{tab:conversion}
\resizebox{0.7\columnwidth}{!}{%
\begin{tabular}{lccc}
\toprule
\textbf{Method} & \textbf{Compute reduction} & \textbf{Measured speedup} & \textbf{Realized} \\
\midrule
SAViT & $3.32\times$ & $2.05\times$ & 62\% \\
NViT & $2.60\times$ & $1.90\times$ & 73\% \\
SAViT & $2.00\times$ & $1.55\times$ & 78\% \\
\textcolor{blue}{S$^2$ViTE} & \textcolor{blue}{$1.50\times$} & \textcolor{blue}{$1.33\times$} & \textcolor{blue}{89\%} \\
\midrule
\textbf{HiAP (Ours)} & $2.38\times$ & $2.22\times$ & \textbf{93\%} \\
\textbf{HiAP (Ours)} & $1.74\times$ & $1.67\times$ & \textbf{96\%} \\
\bottomrule
\end{tabular}
}
\end{table}

Theoretical MAC reductions are only useful if they materialize on hardware. \textcolor{blue}{Masked or unstructured models still run the full dense computation unless sparse kernels are available, while the network HiAP extracts (Section~\ref{subsec:single_phase}) is an ordinary dense ViT.} \textcolor{blue}{We check that it reproduces the logits of the gated network on the same inputs, so accuracy carries over exactly.}

Table~\ref{tab:latency} reports measured latency and throughput against stock dense baselines on a single A100. At 7.4G MACs, HiAP delivers a $2.22\times$ throughput gain over DeiT-Base and reduces single-image latency by 40\%, while retaining 80.88\% top-1; the DeiT-Small sub-networks reach $1.36\times$ and $1.46\times$ throughput.

It is more informative to measure how much of that theoretical compute reduction actually translates into real gain. Table~\ref{tab:conversion} compares the ratio of measured speedup to theoretical compute reduction across methods that have reported both. HiAP realizes over 90\% of its MAC reduction as wall-clock throughput \textcolor{blue}{at batch 256}. \textcolor{blue}{SAViT converts 78\% of its reduction at $2.0\times$ and only 62\% at $3.3\times$, NViT 73\% at $2.6\times$, and S$^2$ViTE 89\% of a smaller $1.5\times$ reduction.}

\section{Conclusion}

HiAP turns ViT pruning from a pipeline of hand-tuned stages into a single budget-aware learning problem. Gumbel-Sigmoid gates at the macro level (heads, FFN blocks) and micro level (value dimensions, FFN neurons) are trained alongside the weights against an analytical MAC cost, so the network settles on its own per-layer allocation without preset sparsity ratios, importance rankings or a separate recovery phase. On ImageNet the resulting sub-networks are competitive with substantially more complex pipelines at matched compute, and because they are physically extracted rather than masked, they run on stock kernels and convert more than 90\% of their theoretical saving into measured throughput \textcolor{blue}{at batch 256}.

Several limitations bound these results. Our micro-gates act on the value path only, so $C_1$, which covers the $Q$\textcolor{blue}{, $K$ and $V$} projections and the attention map, does not shrink when value dimensions are removed. Roughly 70\% of a head's compute is therefore fixed \textcolor{blue}{in our cost model}, and the value-path gates are correspondingly reluctant to engage; extending them to $Q$ and $K$ is the single largest improvement available to the method. The objective targets expected MACs rather than calibrated latency or energy, so realized speedup depends on the kernels involved, and we align extracted widths to multiples of eight for this reason. \textcolor{blue}{Finally, distillation adds about half a point on DeiT-Base.}

Natural extensions are to replace the analytical cost with a measured latency or energy signal, and to combine HiAP with token pruning and quantization.

\section*{Acknowledgments}

Our work was supported by the UKRI EPSRC funding council (EP/V042270/1), a
University of Aberdeen PhD studentship, the HPC facility ”Maxwell”, and in part by the University of
Aberdeen Interdisciplinary Institute’s Alumni and Friends.

%% file: supp_revision.tex
\appendix
\section{Extended Ablation Studies}

\subsection{Early Performance at Various Macro-Micro Settings}
The decoupling of macro and micro-structural penalty is a core contribution of HiAP. To determine the optimal balance between coarse and fine-grained sparsity, we evaluated various penalty ratios over 50 epochs on ImageNet using DeiT-small. We tested macro-to-micro ratios of 2:1, 5:1, and 1.5:1, as well as macro-only and micro-only configurations. The Pareto frontier evaluating the accuracy versus computational cost for these configurations is illustrated in Figure \ref{fig:pareto}.

\begin{figure}[htbp]
    \centering
    \includegraphics[width=0.65\linewidth]{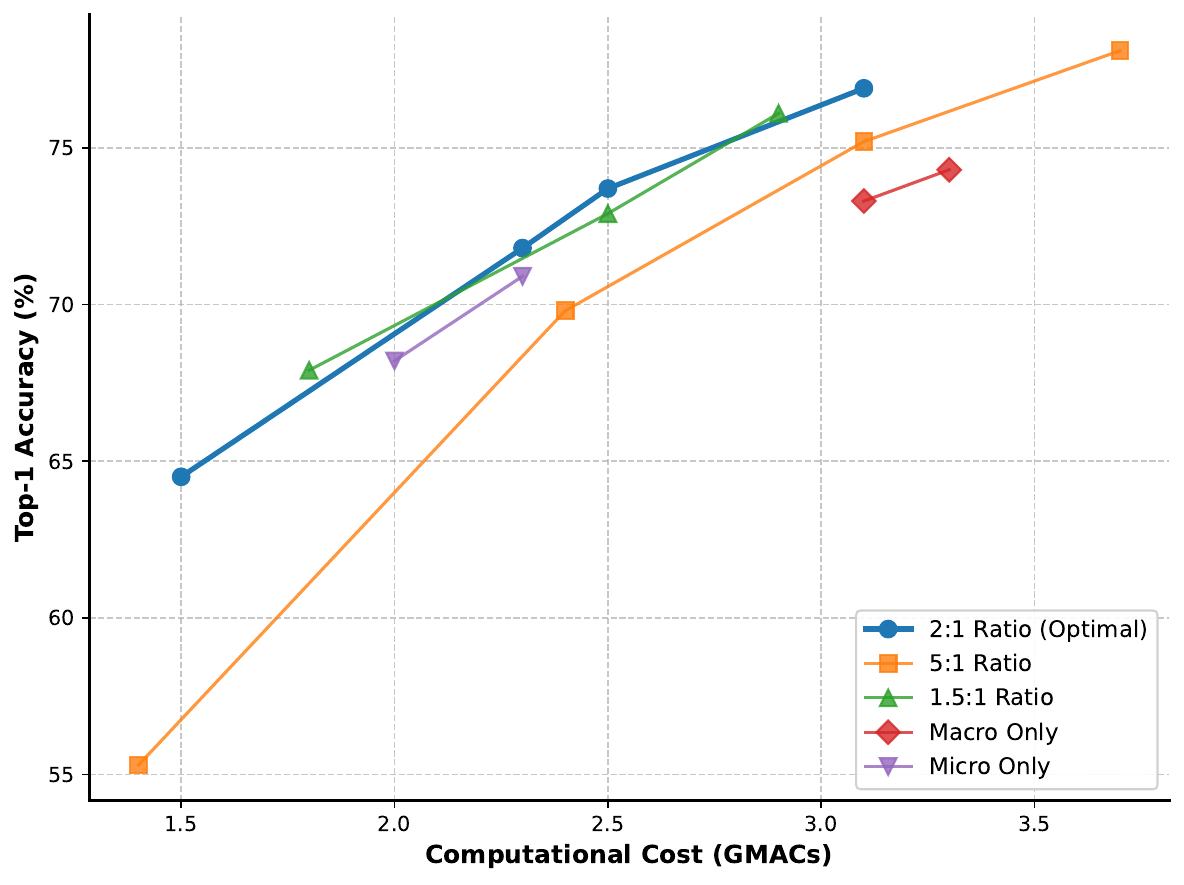}
    \caption{Top-1 Accuracy vs GMACs evaluated at early training stage across different penalty configurations}
    \label{fig:pareto}
\end{figure}

\section{Decomposition of Expected Prunable Cost (Proof of Proposition 1)}

To prove Proposition 1, we show that the expected prunable computational cost of the network can be expressed as an exact linear combination of the expected joint gate states, $\mathbb{E}[C] = \sum_i w_i \, \mathbb{E}[z_i]$. Let $C$ be the total prunable cost of a network. Then $C$ of a ViT of $L$ layers can be decoupled into Attention and FFN components.

\textbf{Attention Cost:} For a given layer $l$, the macro-gate $g_{l,h} \in \{0,1\}$ dictates the retention of an entire attention head $h$. The micro-gate $d_{l,h,j} \in \{0,1\}$ dictates the retention of the $j$-th intra-head dimension. Within a head, the dimensions are only computed if the parent head is active. \textcolor{blue}{An open head pays the fixed cost $C_1$ through $g_{l,h}$ alone, and each surviving dimension adds $C_2$ through the joint state $z_{l,h,j}^{\text{attn}} = g_{l,h} d_{l,h,j}$, with $C_1$ and $C_2$ as defined in Section~\ref{subsec:cost_model}.}

\textbf{FFN Cost:} Similarly, let $b_l \in \{0,1\}$ be the macro-gate for the entire FFN block, and $c_{l,k} \in \{0,1\}$ be the micro-gate for the $k$-th intermediate neuron. The neuron is only active if the block is active. Therefore, the joint variable becomes $z_{l,k}^{\text{ffn}} = b_l c_{l,k}$.

\textcolor{blue}{Each surviving neuron adds $C_3$.} The total discrete prunable cost of the sub-network is the strict sum of the active components:

\[
C =
\sum_{l=1}^{L}
\left(
\sum_{h=1}^{H}
\Big(
\textcolor{blue}{C_1 \, g_{l,h}}
+
\sum_{j=1}^{D_h}
\textcolor{blue}{C_2} \, (g_{l,h} d_{l,h,j})
\Big)
+
\sum_{k=1}^{D_{\text{ffn}}}
\textcolor{blue}{C_3} \, (b_l c_{l,k})
\right).
\]

\textcolor{blue}{Treating $g_{l,h}$ and} the joint variables as the $z_i$, the total prunable cost simplifies to a single linear combination $C = \sum_i w_i z_i$. Taking the expected value gives $\mathbb{E}[C] = \mathbb{E}\left[\sum_i w_i z_i\right]$. By the linearity of expectation, the expected value of a sum equals the sum of the expected values, $\mathbb{E}[X+Y] = \mathbb{E}[X] + \mathbb{E}[Y]$.

This property holds universally for all random variables, regardless of whether they are independent. Therefore, despite the strict hierarchical correlation between the micro-gates and macro-gates, the expectation distributes directly through the summation, $\mathbb{E}[C] = \sum_i w_i \, \mathbb{E}[z_i]$. This confirms that the continuous expected budget constraint directly and linearly approximates the discrete cost without requiring independence assumptions.

\section{Training Complexity}

\begin{table}[h]
\centering
\caption{Pipeline structure of ViT pruning methods. Epoch counts are taken from each work's paper or official implementation and cover the pruning pipeline only; all methods except S$^2$ViTE start from a pre-trained DeiT, which S$^2$ViTE replaces with sparse training from scratch. The rightmost column lists what each pipeline requires beyond ordinary training.}
\label{tab:complexity}
\resizebox{\columnwidth}{!}{%
\begin{tabular}{lccl}
\toprule
\textbf{Method} & \textbf{Stages} & \textbf{Epochs} & \textbf{Requires} \\
\midrule
WDPruning & 1 & 100 & auxiliary shallow classifiers for depth decisions \\
X-Pruner & 2 & $50 + 80$ & explainability-aware mask learning, then layer-wise threshold search \\
UVC & 2 & $30 + 120$ & primal-dual solver over a budget constraint \\
VTP & 3 & $100 + 100$ & $\ell_1$ sparsity training, manual thresholding, then fine-tuning \\
LPViT & 2 & $\text{calib.} + 300$ & calibration-set distortion sampling, block-sparse masks \\
SAViT & 2 & $2 + 300$ & Hessian- and Taylor-based joint importance estimation \\
ViT-Slim & 3 & $50 + 300$ & soft-mask supernet search and ranking \\
S$^2$ViTE & 1 & 600 & periodic prune-and-grow schedule \\
\midrule
\textbf{HiAP (Ours)} & \textbf{1} & \textbf{300} & \textbf{--} \\
\bottomrule
\end{tabular}%
}
\end{table}

\section{Layer Sparsity Allocation}
\begin{table}[h]
\centering
\caption{Discovered topology on DeiT-Base. Heads are out of 12 per layer, FFN neurons out of 3072. Both architectures come from the same objective with only the penalty weights changed \textcolor{blue}{(Table~\ref{tab:hparams})}. Every FFN block survives; compression comes from head removal and FFN width, and the early layers are cut hardest on both.}
\label{tab:topology}
\resizebox{\columnwidth}{!}{%
\begin{tabular}{lcccccccccccc}
\toprule
\textbf{Layer} & 0 & 1 & 2 & 3 & 4 & 5 & 6 & 7 & 8 & 9 & 10 & 11 \\
\midrule
\multicolumn{13}{l}{\textit{10.1G MACs, 81.13\%}} \\
Heads & 2 & 1 & 1 & 3 & 3 & 5 & 5 & 6 & 5 & 4 & 6 & 8 \\
FFN & 1221 & 770 & 1417 & 1968 & 2240 & 2230 & 2476 & 2798 & 2947 & 3007 & 2784 & 1991 \\
\midrule
\multicolumn{13}{l}{\textit{7.4G MACs, 80.88\%}} \\
Heads & 1 & 1 & 1 & 2 & 3 & 2 & 2 & 3 & 1 & 2 & 3 & 5 \\
FFN & 1105 & 770 & 1134 & 1511 & 1769 & 1756 & 1973 & 2357 & 2543 & 2416 & 1799 & 1298 \\
\bottomrule
\end{tabular}%
}
\end{table}

\section{Training Hyperparameters}
\label{app:hparams}

\textcolor{blue}{Table~\ref{tab:hparams} lists every setting used to train the six reported models.} \textcolor{blue}{Settings not listed follow the DeiT recipe.} \textcolor{blue}{On DeiT-Base we trained six penalty settings with distillation and four without; the ones in Table~\ref{tab:base_results} are those that fall in the compute bands of the published baselines.} \textcolor{blue}{The other two runs without distillation, at 11.0G and 12.4G, reach 80.65\% and 80.66\%, within 0.04 points of the 9.1G run, so they add nothing to the comparison.}

\begin{table}[h]
\centering
\caption{\textcolor{blue}{Training hyperparameters for all reported HiAP models.}}
\label{tab:hparams}
\resizebox{\columnwidth}{!}{%
\begin{tabular}{ll}
\toprule
\textcolor{blue}{\textbf{Setting}} & \textcolor{blue}{\textbf{Value}} \\
\midrule
\multicolumn{2}{l}{\textit{Optimization}} \\
\textcolor{blue}{Epochs / batch size} & \textcolor{blue}{300 / 256, single GPU, mixed precision} \\
\textcolor{blue}{Weights} & \textcolor{blue}{AdamW, lr $5\times10^{-5}$, weight decay 0.05, cosine decay, 5 warm-up epochs} \\
\textcolor{blue}{Gate logits} & \textcolor{blue}{AdamW, lr $5\times10^{-3}$, no weight decay, constant} \\
\textcolor{blue}{Seed} & \textcolor{blue}{42, one run per configuration} \\
\midrule
\multicolumn{2}{l}{\textit{Regularization and augmentation}} \\
\textcolor{blue}{Label smoothing / stochastic depth} & \textcolor{blue}{0.1 / 0.1} \\
\textcolor{blue}{RandAugment} & \textcolor{blue}{2 operations, magnitude 9} \\
\textcolor{blue}{Mixup / CutMix} & \textcolor{blue}{$\alpha = 0.8$ / $\alpha = 1.0$, each applied with probability 0.5} \\
\textcolor{blue}{Random erasing} & \textcolor{blue}{0.25} \\
\midrule
\multicolumn{2}{l}{\textit{Gates}} \\
\textcolor{blue}{Initial gate logits $\alpha_0$} & \textcolor{blue}{2.0 for all gates ($\sigma(\alpha_0) \approx 0.88$)} \\
\textcolor{blue}{Temperature} & \textcolor{blue}{$\tau_t = \max(2.0 \cdot 0.98^{t},\, 0.05)$} \\
\textcolor{blue}{Cost penalty} & \textcolor{blue}{one Gumbel-Sigmoid sample per minibatch} \\
\textcolor{blue}{Penalty warm-up} & \textcolor{blue}{$\lambda_{\text{macro}}, \lambda_{\text{micro}}$ ramped linearly over the first 20 epochs} \\
\textcolor{blue}{Hardening} & \textcolor{blue}{$\sigma(\alpha) > 0.5$} \\
\midrule
\multicolumn{2}{l}{\textit{Feasibility penalty}} \\
\textcolor{blue}{$k_{\text{min}}$} & \textcolor{blue}{1 head per layer} \\
\textcolor{blue}{$\gamma_{\text{attn}}$, $\gamma_{\text{ffn}}$} & \textcolor{blue}{0.25} \\
\textcolor{blue}{$\beta_{\text{head}}$, $\beta_{\text{dim}}$, $\beta_{\text{ffn}}$} & \textcolor{blue}{0.01} \\
\midrule
\multicolumn{2}{l}{\textit{Distillation}} \\
\textcolor{blue}{Teacher} & \textcolor{blue}{RegNetY-160, pre-trained; none for the two DeiT-Base runs without distillation} \\
\textcolor{blue}{Temperature $T$} & \textcolor{blue}{4.0} \\
\textcolor{blue}{Distillation weight} & \textcolor{blue}{0.3 (DeiT-Base), 0.7 (DeiT-Small)} \\
\midrule
\multicolumn{2}{l}{\textit{Penalty weights $(\lambda_{\text{macro}}, \lambda_{\text{micro}})$}} \\
\textcolor{blue}{DeiT-Base, 10.1G / 7.4G} & \textcolor{blue}{(0.9, 0.45) / (2.2, 1.1)} \\
\textcolor{blue}{DeiT-Base without distillation, 9.1G / 7.8G} & \textcolor{blue}{(1.5, 0.75) / (2.2, 1.1)} \\
\textcolor{blue}{DeiT-Small, 3.1G / 2.9G} & \textcolor{blue}{(0.4, 0.2) / (0.9, 0.45)} \\
\bottomrule
\end{tabular}%
}
\end{table}

%% file: main.bib
@String(CVPR  = {IEEE Conf. Comput. Vis. Pattern Recog.})

@String(AAAI  = {AAAI})

@String(ICASSP=	{ICASSP})

@String(CVPR  = {CVPR})

@article{tang2020scop,
  title={Scop: Scientific control for reliable neural network pruning},
  author={Tang, Yehui and Wang, Yunhe and Xu, Yixing and Tao, Dacheng and Xu, Chunjing and Xu, Chao and Xu, Chang},
  journal={Advances in neural information processing systems},
  volume={33},
  pages={10936--10947},
  year={2020}
}

@article{michel2019sixteen,
  title={Are sixteen heads really better than one?},
  author={Michel, Paul and Levy, Omer and Neubig, Graham},
  journal={Advances in neural information processing systems},
  volume={32},
  year={2019}
}

@inproceedings{kong2023peeling,
  title={Peeling the onion: Hierarchical reduction of data redundancy for efficient vision transformer training},
  author={Kong, Zhenglun and Ma, Haoyu and Yuan, Geng and Sun, Mengshu and Xie, Yanyue and Dong, Peiyan and Meng, Xin and Shen, Xuan and Tang, Hao and Qin, Minghai and others},
  booktitle={Proceedings of the AAAI Conference on Artificial Intelligence},
  volume={37},
  pages={8360--8368},
  year={2023}
}

@inproceedings{molchanov2019importance,
  title={Importance estimation for neural network pruning},
  author={Molchanov, Pavlo and Mallya, Arun and Tyree, Stephen and Frosio, Iuri and Kautz, Jan},
  booktitle={Proceedings of the IEEE/CVF conference on computer vision and pattern recognition},
  pages={11264--11272},
  year={2019}
}

@article{chen2021chasing,
  title={Chasing sparsity in vision transformers: An end-to-end exploration},
  author={Chen, Tianlong and Cheng, Yu and Gan, Zhe and Yuan, Lu and Zhang, Lei and Wang, Zhangyang},
  journal={Advances in Neural Information Processing Systems},
  volume={34},
  pages={19974--19988},
  year={2021}
}

@inproceedings{yang2023global,
  title={Global vision transformer pruning with hessian-aware saliency},
  author={Yang, Huanrui and Yin, Hongxu and Shen, Maying and Molchanov, Pavlo and Li, Hai and Kautz, Jan},
  booktitle={Proceedings of the IEEE/CVF conference on computer vision and pattern recognition},
  pages={18547--18557},
  year={2023}
}

@inproceedings{yin2023gohsp,
  title={Gohsp: A unified framework of graph and optimization-based heterogeneous structured pruning for vision transformer},
  author={Yin, Miao and Uzkent, Burak and Shen, Yilin and Jin, Hongxia and Yuan, Bo},
  booktitle={Proceedings of the AAAI Conference on Artificial Intelligence},
  volume={37},
  pages={10954--10962},
  year={2023}
}

@article{yu2019autoslim,
  title={Autoslim: Towards one-shot architecture search for channel numbers},
  author={Yu, Jiahui and Huang, Thomas},
  journal={arXiv preprint arXiv:1903.11728},
  year={2019}
}

@article{rao2021dynamicvit,
  title={Dynamicvit: Efficient vision transformers with dynamic token sparsification},
  author={Rao, Yongming and Zhao, Wenliang and Liu, Benlin and Lu, Jiwen and Zhou, Jie and Hsieh, Cho-Jui},
  journal={Advances in neural information processing systems},
  volume={34},
  pages={13937--13949},
  year={2021}
}

@article{liang2022not,
  title={Not all patches are what you need: Expediting vision transformers via token reorganizations},
  author={Liang, Youwei and Ge, Chongjian and Tong, Zhan and Song, Yibing and Wang, Jue and Xie, Pengtao},
  journal={arXiv preprint arXiv:2202.07800},
  year={2022}
}

@inproceedings{yin2022vit,
  title={A-vit: Adaptive tokens for efficient vision transformer},
  author={Yin, Hongxu and Vahdat, Arash and Alvarez, Jose M and Mallya, Arun and Kautz, Jan and Molchanov, Pavlo},
  booktitle={Proceedings of the IEEE/CVF conference on computer vision and pattern recognition},
  pages={10809--10818},
  year={2022}
}

@article{bolya2022token,
  title={Token merging: Your vit but faster},
  author={Bolya, Daniel and Fu, Cheng-Yang and Dai, Xiaoliang and Zhang, Peizhao and Feichtenhofer, Christoph and Hoffman, Judy},
  journal={arXiv preprint arXiv:2210.09461},
  year={2022}
}

@article{bengio2015conditional,
  title={Conditional computation in neural networks for faster models},
  author={Bengio, Emmanuel and Bacon, Pierre-Luc and Pineau, Joelle and Precup, Doina},
  journal={arXiv preprint arXiv:1511.06297},
  year={2015}
}

@article{louizos2017learning,
  title={Learning sparse neural networks through $ L\_0 $ regularization},
  author={Louizos, Christos and Welling, Max and Kingma, Diederik P},
  journal={arXiv preprint arXiv:1712.01312},
  year={2017}
}

@inproceedings{yu2023x,
  title={X-pruner: explainable pruning for vision transformers},
  author={Yu, Lu and Xiang, Wei},
  booktitle={Proceedings of the IEEE/CVF conference on computer vision and pattern recognition},
  pages={24355--24363},
  year={2023}
}

@article{cai2018proxylessnas,
  title={Proxylessnas: Direct neural architecture search on target task and hardware},
  author={Cai, Han and Zhu, Ligeng and Han, Song},
  journal={arXiv preprint arXiv:1812.00332},
  year={2018}
}

@inproceedings{kong2022spvit,
  title={Spvit: Enabling faster vision transformers via latency-aware soft token pruning},
  author={Kong, Zhenglun and Dong, Peiyan and Ma, Xiaolong and Meng, Xin and Niu, Wei and Sun, Mengshu and Shen, Xuan and Yuan, Geng and Ren, Bin and Tang, Hao and others},
  booktitle={European conference on computer vision},
  pages={620--640},
  year={2022},
  organization={Springer}
}

@article{dosovitskiy2020image,
  title={An image is worth 16x16 words: Transformers for image recognition at scale},
  author={Dosovitskiy, Alexey and Beyer, Lucas and Kolesnikov, Alexander and Weissenborn, Dirk and Zhai, Xiaohua and Unterthiner, Thomas and Dehghani, Mostafa and Minderer, Matthias and Heigold, Georg and Gelly, Sylvain and others},
  journal={arXiv preprint arXiv:2010.11929},
  year={2020}
}

@inproceedings{chavan2022vision,
  title={Vision transformer slimming: Multi-dimension searching in continuous optimization space},
  author={Chavan, Arnav and Shen, Zhiqiang and Liu, Zhuang and Liu, Zechun and Cheng, Kwang-Ting and Xing, Eric P},
  booktitle={Proceedings of the IEEE/CVF conference on computer vision and pattern recognition},
  pages={4931--4941},
  year={2022}
}

@article{han2015deep,
  title={Deep compression: Compressing deep neural networks with pruning, trained quantization and huffman coding},
  author={Han, Song and Mao, Huizi and Dally, William J},
  journal={arXiv preprint arXiv:1510.00149},
  year={2015}
}

@article{dao2022flashattention,
  title={Flashattention: Fast and memory-efficient exact attention with io-awareness},
  author={Dao, Tri and Fu, Dan and Ermon, Stefano and Rudra, Atri and R{\'e}, Christopher},
  journal={Advances in neural information processing systems},
  volume={35},
  pages={16344--16359},
  year={2022}
}

@inproceedings{liu2024updp,
  title={Updp: A unified progressive depth pruner for cnn and vision transformer},
  author={Liu, Ji and Tang, Dehua and Huang, Yuanxian and Zhang, Li and Zeng, Xiaocheng and Li, Dong and Lu, Mingjie and Peng, Jinzhang and Wang, Yu and Jiang, Fan and others},
  booktitle={Proceedings of the AAAI conference on artificial intelligence},
  volume={38},
  pages={13891--13899},
  year={2024}
}

@article{zheng2022savit,
  title={Savit: Structure-aware vision transformer pruning via collaborative optimization},
  author={Zheng, Chuanyang and Zhang, Kai and Yang, Zhi and Tan, Wenming and Xiao, Jun and Ren, Ye and Pu, Shiliang and others},
  journal={Advances in Neural Information Processing Systems},
  volume={35},
  pages={9010--9023},
  year={2022}
}

@inproceedings{xu2024lpvit,
  title={Lpvit: Low-power semi-structured pruning for vision transformers},
  author={Xu, Kaixin and Wang, Zhe and Chen, Chunyun and Geng, Xue and Lin, Jie and Yang, Xulei and Wu, Min and Li, Xiaoli and Lin, Weisi},
  booktitle={European Conference on Computer Vision},
  pages={269--287},
  year={2024},
  organization={Springer}
}

@article{sun2025mdp,
  title={MDP: Multidimensional Vision Model Pruning with Latency Constraint},
  author={Sun, Xinglong and Lakshmanan, Barath and Shen, Maying and Lan, Shiyi and Chen, Jingde and Alvarez, Jose M},
  journal={arXiv preprint arXiv:2504.02168},
  year={2025}
}

@inproceedings{fang2024isomorphic,
  title={Isomorphic pruning for vision models},
  author={Fang, Gongfan and Ma, Xinyin and Mi, Michael Bi and Wang, Xinchao},
  booktitle={European Conference on Computer Vision},
  pages={232--250},
  year={2024},
  organization={Springer}
}

@ARTICLE{8485719,
  author={Chen, Shi and Zhao, Qi},
  journal={IEEE Transactions on Pattern Analysis and Machine Intelligence}, 
  title={Shallowing Deep Networks: Layer-Wise Pruning Based on Feature Representations}, 
  year={2019},
  volume={41},
  number={12},
  pages={3048-3056},
  doi={10.1109/TPAMI.2018.2874634}}

@article{wang2019dbp,
  title={Dbp: Discrimination based block-level pruning for deep model acceleration},
  author={Wang, Wenxiao and Zhao, Shuai and Chen, Minghao and Hu, Jinming and Cai, Deng and Liu, Haifeng},
  journal={arXiv preprint arXiv:1912.10178},
  year={2019}
}

@article{xu2020layer,
  title={Layer pruning via fusible residual convolutional block for deep neural networks},
  author={Xu, Pengtao and Cao, Jian and Shang, Fanhua and Sun, Wenyu and Li, Pu},
  journal={arXiv preprint arXiv:2011.14356},
  year={2020}
}

@article{hassibi1992second,
  title={Second order derivatives for network pruning: Optimal brain surgeon},
  author={Hassibi, Babak and Stork, David},
  journal={Advances in neural information processing systems},
  volume={5},
  year={1992}
}

@article{zhu2021vision,
  title={Vision transformer pruning},
  author={Zhu, Mingjian and Tang, Yehui and Han, Kai},
  journal={arXiv preprint arXiv:2104.08500},
  year={2021}
}

@article{yu2022unified,
  title={Unified visual transformer compression},
  author={Yu, Shixing and Chen, Tianlong and Shen, Jiayi and Yuan, Huan and Tan, Jianchao and Yang, Sen and Liu, Ji and Wang, Zhangyang},
  journal={International Conference on Learning Representations},
  year={2022}
}

@InProceedings{Fang_2023_CVPR,
    author    = {Fang, Gongfan and Ma, Xinyin and Song, Mingli and Mi, Michael Bi and Wang, Xinchao},
    title     = {DepGraph: Towards Any Structural Pruning},
    booktitle = {Proceedings of the IEEE/CVF Conference on Computer Vision and Pattern Recognition (CVPR)},
    month     = {June},
    year      = {2023},
    pages     = {16091-16101}
}

@inproceedings{tang2023sr,
  title={Sr-init: An interpretable layer pruning method},
  author={Tang, Hui and Lu, Yao and Xuan, Qi},
  booktitle={ICASSP 2023-2023 IEEE International Conference on Acoustics, Speech and Signal Processing (ICASSP)},
  pages={1--5},
  year={2023},
  organization={IEEE}
}

@article{fan2019reducing,
  title={Reducing transformer depth on demand with structured dropout},
  author={Fan, Angela and Grave, Edouard and Joulin, Armand},
  journal={arXiv preprint arXiv:1909.11556},
  year={2019}
}

@inproceedings{yu2022width,
  title={Width \& depth pruning for vision transformers},
  author={Yu, Fang and Huang, Kun and Wang, Meng and Cheng, Yuan and Chu, Wei and Cui, Li},
  booktitle={Proceedings of the AAAI conference on artificial intelligence},
  volume={36},
  pages={3143--3151},
  year={2022}
}

@inproceedings{tang2022patch,
  title={Patch Slimming for Efficient Vision Transformers},
  author={Tang, Yehui and Han, Kai and Wang, Yunhe and Xu, Chang and Guo, Jianyuan and Xu, Chao and Tao, Dacheng},
  booktitle={Proceedings of the IEEE/CVF Conference on Computer Vision and Pattern Recognition (CVPR)},
  pages={12165--12174},
  year={2022}
}

@inproceedings{xia2022structured,
  title={Structured Pruning Learns Compact and Accurate Models},
  author={Xia, Mengzhou and Zhong, Zexuan and Chen, Danqi},
  booktitle={Proceedings of the 60th Annual Meeting of the Association for Computational Linguistics (Volume 1: Long Papers)},
  pages={1513--1528},
  year={2022}
}
